\documentclass[preprint,12pt]{elsarticle}

\usepackage{amssymb, amsthm, enumerate, graphicx, amscd, amsmath, amssymb, amsfonts, xcolor, bm}

\usepackage{enumitem}
\usepackage{hyperref}
\usepackage{tikz}
\usetikzlibrary{arrows.meta,positioning}
\usepackage{booktabs}
\usepackage{float}
\usepackage{hyperref}
\newcommand{\doi}[1]{\href{https://doi.org/#1}{DOI: #1}}
\usepackage[section]{placeins}

\numberwithin{equation}{section}

\newtheorem{proposition}[equation]{Proposition}

\newtheorem{example}[equation]{Example}

\journal{}

\begin{document}

\begin{frontmatter}



\title{Machine Learning Model for Sparse PCM Completion}

\author[coppin]{Selcuk Koyuncu\corref{cor}}
\ead{skoyuncu1@gmail.com}

\author[coppin]{Ronak Nouri}
\ead{rnouri@coppin.edu}

\author[coppin]{Stephen Providence}
\ead{sprovidence@coppin.edu}

\cortext[cor]{Corresponding author.}

\address[coppin]{Department of Mathematics and Computer Science,
Coppin State University, Baltimore, MD, USA}

\begin{abstract}
In this paper, we propose a machine learning model for sparse pairwise comparison matrices (PCMs), combining classical PCM approaches with graph-based learning techniques. Numerical results are provided to demonstrate the effectiveness and scalability of the proposed method.
\end{abstract}


\end{frontmatter}

\section{Introduction}
Let $A = (a_{ij})_{n \times n}$ be a square matrix of order $n$ associated with 
alternatives $A_1, A_2, \dots, A_n$. Then $A$ is called a 
\emph{pairwise comparison matrix (PCM)} if it satisfies:
\begin{enumerate}
    \item $a_{ij} > 0$ for all $i,j$,
    \item $a_{ii} = 1$ for all $i$,
    \item $a_{ij} = \tfrac{1}{a_{ji}}$ for all $i \neq j$.
\end{enumerate}
Here $a_{ij}$ represents the relative preference or importance of alternative 
$A_i$ over $A_j$. The matrix $A$ is said to be \emph{consistent} if it further 
satisfies
\[
a_{ij} \cdot a_{jk} = a_{ik}, \quad \text{for all } i,j,k.
\]

Pairwise comparisons arise whenever human judgments or observed interactions are inherently relative rather than absolute.
Examples include preference elicitation in decision analysis, best--worst style questioning, crowdsourcing and peer grading, sports and tournament outcomes, and implicit feedback in recommender systems.
In such settings, obtaining all $\binom{n}{2}$ comparisons is infeasible, and one typically observes only a small subset $|\Omega| \ll n^2$.

While classical AHP applications often involve moderate $n$, large-scale regimes naturally occur in modern data-driven applications:
for instance, items may correspond to products, web pages, or users, where $n$ may reach $10^5$--$10^7$ and the observed comparisons scale closer to $O(n\log n)$ or even $O(n)$.
In these regimes, the central computational challenge is to infer missing comparisons (matrix completion) and extract a reliable global ranking from sparse, noisy, and possibly inconsistent observations, while maintaining the reciprocal structure of PCMs and promoting multiplicative consistency.
This motivates scalable learning algorithms that operate directly on sparse graphs rather than dense $n\times n$ matrices.

Pairwise comparison matrices (PCMs) were introduced by Saaty in the late 1970s 
as the foundation of the Analytic Hierarchy Process (AHP) \cite{saaty1980}. A PCM 
encodes relative preferences among alternatives, with consistency holding when 
there exists a weight vector $w$ such that $a_{ij}=w_i/w_j$ for all $i,j$. In 
practice, human judgments are rarely consistent, and various indices have been 
proposed to measure inconsistency, including Saaty’s consistency ratio and 
Koczkodaj’s triad-based index \cite{koczkodaj1993}. Methods for deriving 
priorities from PCMs include the eigenvector method, the geometric mean method 
(which is equivalent to solving a log least squares problem \cite{crawford1985}), 
and statistical approaches linked to Bradley--Terry--Luce and Thurstone models.  

When PCMs are inconsistent, one natural task is to find the nearest consistent 
matrix, a problem that is NP-hard in general \cite{bozoki2010}. Approximation 
methods, such as convex relaxations and least squares formulations, have been 
proposed to address this challenge. More recently, connections have been drawn 
between PCMs and graph-based ranking methods. Spectral algorithms such as Rank 
Centrality \cite{negahban2012} use Markov chain techniques to recover rankings 
with provable guarantees, and results in learning theory show that accurate 
rankings can often be obtained from only $O(n \log n)$ pairwise comparisons.  

Beyond classical decision analysis, PCMs have been applied to group decision 
making, recommender systems, sports analytics, and crowdsourcing, with 
extensions to interval and fuzzy settings. Despite these advances, most existing 
theory assumes that PCMs are dense. In reality, PCMs are often highly sparse 
because it is infeasible to collect all $\binom{n}{2}$ comparisons when the 
number of alternatives is large. Understanding how to complete sparse PCMs, 
control inconsistency, and extract robust global rankings is therefore an 
emerging and important research direction that lies at the intersection of 
matrix theory, optimization, and machine learning.

While consistency measures such as Consistency Index (CI) and Consistency Ratio (CR) provide useful diagnostics, they do not directly address the problem of completing sparse or highly inconsistent PCMs.
In this work, we introduce a machine learning model that reconstructs missing comparisons while maintaining multiplicative consistency.

 We present the following simple example to remind the reader of the fundamental idea.

Consider three alternatives $A_1, A_2, A_3$. Suppose their true weights are 
$w = (2,\,1,\,0.5)$. Then the pairwise comparison matrix is
\[
A =
\begin{bmatrix}
1 & 2 & 4 \\
1/2 & 1 & 2 \\
1/4 & 1/2 & 1
\end{bmatrix}.
\]
This matrix is \emph{consistent} because $a_{ij} \cdot a_{jk} = a_{ik}$ for all 
$i,j,k$. For instance, $a_{12}\cdot a_{23} = 2 \cdot 2 = 4 = a_{13}$.

Now suppose we are given the following judgments:
\[
B =
\begin{bmatrix}
1 & 3 & 4 \\
1/3 & 1 & 2 \\
1/4 & 1/2 & 1
\end{bmatrix}.
\]
This matrix is \emph{inconsistent}, because 
$a_{12}\cdot a_{23} = 3 \cdot 2 = 6 \neq a_{13} = 4$. 
Such inconsistencies are typical in real-world decision-making, 
where human judgments are not perfectly transitive.

Consider
\[
B =
\begin{bmatrix}
1 & 3 & 4 \\
\frac{1}{3} & 1 & 2 \\
\frac{1}{4} & \frac{1}{2} & 1
\end{bmatrix}.
\]

\paragraph{Step 1: Priority Vector.}
Compute the principal right eigenvector $w$ of $B$ (normalized to sum to 1). Numerically,
\[
w \approx (0.6250,\; 0.2385,\; 0.1365)^T.
\]

\paragraph{Step 2: Maximum Eigenvalue.}
Let $\lambda_{\max}$ be the largest eigenvalue of $B$. Numerically,
\[
\lambda_{\max} \approx 3.0183.
\]

\paragraph{Step 3: Consistency Index (CI).}
For $n=3$,
\[
CI \;=\; \frac{\lambda_{\max}-n}{n-1}
\;=\; \frac{3.0183-3}{2}
\;\approx\; 0.00915.
\]

\paragraph{Step 4: Consistency Ratio (CR)}
Using Saaty’s Random Index $RI(3)=0.58$,
\[
CR \;=\; \frac{CI}{RI}
\;=\; \frac{0.00915}{0.58}
\;\approx\; 0.0158.
\]

Since $CR \approx 0.016 < 0.10$, matrix $B$ is not perfectly consistent, for instance
 $a_{12}a_{23}=6 \neq a_{13}=4$, but considered \emph{acceptably consistent}
under Saaty’s criterion.

The goal of this work is to complete sparse reciprocal PCMs and recover global rankings in a way that (i) leverages only the observed comparisons, (ii) scales to large graphs, and (iii) explicitly promotes multiplicative consistency through a triangle-based regularization in log space.
Our approach combines classical PCM ideas (log-ratio modeling, reciprocity, and transitivity) with graph-based representation learning, yielding a practical completion method that is comparable in accuracy to standard baselines on synthetic benchmarks and remains computationally feasible in sparse large-scale regimes.

A sparse pairwise comparison matrix (PCM) is a partially filled matrix with many missing entries, which raises the challenge of reconstructing the full comparison structure and extracting consistent rankings. In this paper, we introduce a blended machine learning model to find the missing entries and enforce consistency for PCMs of arbitrary size. For instance, the symbols “?” denote missing entries in the matrix, and our method provides a systematic procedure to complete them while maintaining pairwise consistency:
\[
A = \begin{bmatrix}
1 & 3 & ? & ? & ? \\[4pt]
\dfrac{1}{3} & 1 & 5 & ? & ? \\[4pt]
? & \dfrac{1}{5} & 1 & 2 & ? \\[4pt]
? & ? & \dfrac{1}{2} & 1 & 4 \\[4pt]
? & ? & ? & \dfrac{1}{4} & 1
\end{bmatrix}.
\]

A substantial literature studies incomplete PCMs and how many comparisons are needed to derive reliable priorities.
In particular, recent work has investigated when incompleteness can be tolerated without inducing excessive inconsistency or instability.
Ag\'oston and Csat\'o~\cite{agoston2022} study inconsistency thresholds for incomplete PCMs and provide conditions under which incomplete judgments remain meaningful.
Brunelli~\cite{brunelli2022} discusses limitations of using only $(n-1)$ comparisons and reviews why additional comparisons may be necessary for robustness.
These contributions emphasize that sparsity is not only a computational issue but also a statistical and decision-theoretic one: different observation patterns can yield very different uncertainty and inconsistency behavior.

Sparse PCM completion is closely related to matrix completion and low-rank factorization methods widely used in machine learning, especially in recommender systems.
Classical approaches approximate a partially observed matrix by a low-rank factorization and estimate latent factors via regularized optimization.
In the pairwise setting, log-ratio parametrizations such as $a_{ij}\approx \exp(x_i-x_j)$ can be viewed as a structured (rank-two) factorization in log space.
Our graph-based model can be interpreted as a nonlinear extension in which node embeddings replace scalar scores and are refined through message passing on the observed comparison graph.

Another related line concerns probabilistic ranking and choice models (Bradley--Terry, Plackett--Luce, and related families) fitted from partial comparisons.
Such models provide a principled way to aggregate sparse outcomes into global scores and induce consistent preference probabilities.
For example, Cheng et al.~\cite{cheng2012} discuss label ranking with partial abstention using thresholded probabilistic models, illustrating how learning a global model and thresholding can yield coherent partial preferences.
Our framework is compatible with this viewpoint: the BTL-mode loss fits observed outcomes, while the consistency regularizer promotes transitivity in the induced preference structure.

The paper is organized as follows. 
Sections~1 and~2 present summaries of two well-known baseline methods for pairwise comparison matrices. 
In Section~4, we introduce the proposed machine learning model and include several illustrative examples. 
Section~5 reports numerical experiments demonstrating the performance of the model. 
Finally, Section~6 discusses scalability and efficiency for handling large-scale graphs, along with additional numerical results.

We now present two well-known models.

\section{The Bradley--Terry--Luce (BTL) Model}

The Bradley--Terry--Luce (BTL) model \cite{bradley1952rank,luce1959individual} is a classical probabilistic framework for pairwise comparisons. 
Suppose there are $n$ alternatives $A_1, \dots, A_n$. Each alternative $A_i$ is assigned a latent score $x_i \in \mathbb{R}$. 
The probability that $A_i$ is preferred to $A_j$ is modeled as
\[
    \Pr(A_i \succ A_j) 
    = \sigma(x_i - x_j) 
    = \frac{e^{x_i}}{e^{x_i}+e^{x_j}}, \quad i \neq j,
\]
where $\sigma(\cdot)$ denotes the logistic sigmoid function.

\subsection*{Maximum Likelihood Estimation}
Given observed outcomes of pairwise comparisons, the latent scores ${x} = (x_1,\dots,x_n)$ are estimated by maximizing the log-likelihood
\[
    \ell({x}) 
    = \sum_{i \neq j} \Big[ c_{ij}\,\log \sigma(x_i - x_j) 
    + c_{ji}\,\log \sigma(x_j - x_i) \Big],
\]
where $c_{ij}$ is the number of times $A_i$ beats $A_j$. 
This optimization is convex (up to an additive constant, since ${x}$ is identifiable only up to a global shift). 
Typically, one enforces the constraint $\sum_i x_i = 0$ to fix the gauge.

\section{The Log--Least--Squares (LLS) Method}

While the BTL model is probabilistic, the Log--Least--Squares (LLS) method is a deterministic approach for completing and approximating pairwise comparison matrices (PCMs). 
Suppose $A = (a_{ij})$ is a PCM with some observed entries, where $a_{ij} > 0$ and $a_{ij} = 1/a_{ji}$. 
The LLS method assumes the existence of latent scores ${x} = (x_1,\dots,x_n)$ such that
\[
    a_{ij} \approx e^{x_i - x_j}.
\]
The latent scores are estimated by solving the convex optimization problem
\[
    \min_{{x} \in \mathbb{R}^n} 
    \sum_{(i,j)\in \Omega} \big( x_i - x_j - \log a_{ij} \big)^2,
\]
where $\Omega$ denotes the set of observed entries. 
This reduces to solving a linear system involving the graph Laplacian induced by $\Omega$.

\section{A Machine Learning Model for Sparse PCM Completion}

We model a sparse pairwise comparison matrix (PCM) as a graph $G=(V,E)$ with $|V|=n$ items and observed comparisons on edges $E\subseteq V\times V$. 
Let $\Omega\subseteq\{(i,j):i\neq j\}$ denote observed directed comparisons; each $(i,j)\in\Omega$ carries either
(i) a \emph{binary} outcome $y_{ij}\in\{0,1\}$ (\emph{BTL mode}), or 
(ii) a \emph{cardinal} ratio $a_{ij}>0$ (\emph{LLS mode}). 
We work in log space with targets $t_{ij}=\log a_{ij}$ for cardinal data.

 We first explain latent embeddings and edge predictor, 
each item $i$ has a learnable embedding $h_i\in\mathbb{R}^d$. A shallow message-passing network updates embeddings by aggregating neighbors:
\[
h^{(0)}_i \in \mathbb{R}^d,\qquad 
h^{(\ell+1)}_i=\phi\!\Big(W_1\,h^{(\ell)}_i+\!\!\!\sum_{j\in\mathcal{N}(i)}\!\!\! W_2\,h^{(\ell)}_j\Big),\quad \ell=0,\dots,L-1,
\]
with nonlinearity $\phi$ and learned weights $W_1,W_2$.
Given final embeddings $h_i=h^{(L)}_i$, we predict the (log) comparison via a linear edge head on the \emph{difference}:
\[
\widehat{t}_{ij} = v^\top(h_i-h_j)\quad\Longrightarrow\quad \widehat{a}_{ij}=\exp(\widehat{t}_{ij}),\qquad
\widehat{p}_{ij}=\sigma\big(v^\top(h_i-h_j)\big),
\]
where $\widehat{p}_{ij}$ is the probability $\Pr(i\succ j)$ used in BTL mode, $v \in \mathbb{R}^d$ is a learnable weight vector of dimension $d,$
 and $\sigma$ is the logistic sigmoid.

For training objectives,
we fit observed edges and enforce global multiplicative consistency:
\[
\mathcal{L}_{\text{data}}=
\begin{cases}
-\sum_{(i,j)\in \Omega}\big[y_{ij}\log \widehat{p}_{ij}+(1-y_{ij})\log(1-\widehat{p}_{ij})\big], & \text{BTL mode},\\[2mm]
\sum_{(i,j)\in \Omega}\big(\widehat{t}_{ij}-t_{ij}\big)^2, & \text{LLS mode},
\end{cases}
\]
\[
\mathcal{L}_{\triangle}=\frac{1}{|\mathcal{T}|}\sum_{(i,j,k)\in \mathcal{T}}
\Big|\widehat{t}_{ij}+\widehat{t}_{jk}-\widehat{t}_{ik}\Big|,
\qquad
\mathcal{L}_{\text{reg}}=\|W_1\|_F^2+\|W_2\|_F^2+\|v\|_2^2.
\]
Here $\mathcal{T}$ is a random set of node triples (triangle sampling) and $\widehat{t}_{ij}=\log \widehat{a}_{ij}$ in both modes. 
The total loss is
\[
\mathcal{L} = \mathcal{L}_{\text{data}} + \lambda_{\triangle}\,\mathcal{L}_{\triangle} + \lambda_{\text{reg}}\,\mathcal{L}_{\text{reg}}.
\]
Minimizing $\mathcal{L}_{\triangle}$ encourages \emph{additive transitivity} in log space, i.e., multiplicative consistency of ratios.

For next step, we can do prediction and projection.
After training, we predict missing entries by $\widehat{a}_{ij}=\exp(v^\top(h_i-h_j))$ (LLS mode) or odds $\widehat{a}_{ij}=\widehat{p}_{ij}/(1-\widehat{p}_{ij})$ (BTL mode).
We then \emph{project} to an exactly reciprocal PCM by setting $\widehat{a}_{ii}=1$ and replacing each pair $(\widehat{a}_{ij},\widehat{a}_{ji})$ with the nearest reciprocal pair via the geometric-mean projection:
\[
\widetilde{a}_{ij}\leftarrow \sqrt{\widehat{a}_{ij}\cdot \frac{1}{\widehat{a}_{ji}}},\qquad
\widetilde{a}_{ji}\leftarrow \frac{1}{\widetilde{a}_{ij}}.
\]

Finally, we look at Computation and scalability.
Each epoch costs $O(|\Omega|d)$ for data loss and $O(|\mathcal{T}|d)$ for triangle sampling (choose $|\mathcal{T}| \ll n^3$). 
For very large $n$, subsample edges per batch and compute the dense all-pairs matrix only at the end.

We now present an example to illustrate the details of our method.

\begin{example}
We give the latent embedding and edge prediction mechanism with a simple
pairwise comparison graph of $n=4$ items 
$A_1, A_2, A_3, A_4$. Suppose we observe the comparisons
\[
a_{12}=3,\qquad a_{23}=5,\qquad a_{34}=2,
\]

In this case, the pairwise comparison matrix is
\[
A = \begin{bmatrix}
1     & 3     & ?     & ? \\
\dfrac{1}{3} & 1     & 5     & ? \\
?     & \dfrac{1}{5} & 1     & 2 \\
?     & ?     & \dfrac{1}{2} & 1
\end{bmatrix}
\]

and the observed edge set is
$\Omega=\{(1,2),(2,3),(3,4)\}$.  The goal is to assign values to the missing entries in such a way that the completed matrix satisfies pairwise consistency.

\paragraph{Step 1:} To initiate it, 
each item $i$ is assigned an initial embedding vector 
$h_i^{(0)} \in \mathbb{R}^2$. For illustration we set
\[
h_1^{(0)}=(0.2,\,-0.1),\quad 
h_2^{(0)}=(-0.3,\,0.4),\quad 
h_3^{(0)}=(0.1,\,0.0),\quad 
h_4^{(0)}=(-0.2,\,-0.5).
\]

\paragraph{Step 2:} For  message passing update,
we update each embedding by
\[
h_i^{(1)} = \phi\!\Big( W_1 h_i^{(0)} + \!\!\!\!\sum_{j\in \mathcal{N}(i)} \!\!\!\! W_2 h_j^{(0)}\Big),
\]
with $W_1=I$, $W_2=0.5I$.  

\emph{For node $A_2$:}  
\[
h_2^{(1)} = \phi\!\big( h_2^{(0)} + 0.5(h_1^{(0)}+h_3^{(0)}) \big).
\]
Numerically,
\begin{equation}
\begin{array}{rcl}
h_2^{(1)} 
&=& \phi\!\big( (-0.3,0.4) + 0.5((0.2,-0.1)+(0.1,0.0)) \big) \\[6pt]
&=& \phi\!\big( (-0.3,0.4) + (0.15,-0.05) \big) \\[6pt]
&=& \phi(-0.15,\,0.35). \nonumber
\end{array}
\end{equation}

\paragraph{Step 3:}
ReLU is $\phi(x)=\max(0,x)$ applied componentwise.
\[
h_2^{(1)} = \phi(-0.15,\,0.35) = (0,\,0.35).
\]

Similarly:
\begin{align*}
h_1^{(1)} &= \phi\!\big( (0.2,-0.1) + 0.5(-0.3,0.4) \big) 
= \phi(0.05,\,0.1) = (0.05,\,0.1),\\[6pt]
h_3^{(1)} &= \phi\!\big( (0.1,0.0) + 0.5((-0.3,0.4)+(-0.2,-0.5)) \big) 
= \phi((-0.35,-0.05)) = (0,\,0),\\[6pt]
h_4^{(1)} &= \phi\!\big( (-0.2,-0.5) + 0.5(0.1,0.0) \big) 
= \phi((-0.15,-0.5)) = (0,\,0).
\end{align*}

So the updated embeddings are
\[
h_1^{(1)}=(0.05,\,0.1),\quad
h_2^{(1)}=(0,\,0.35),\quad
h_3^{(1)}=(0,\,0),\quad
h_4^{(1)}=(0,\,0).
\]

\paragraph{Step 4:} For the edge prediction,
take $v=(1,1)^\top$,   
then
\[
\widehat{t}_{ij} = v^\top(h_i-h_j), \qquad \widehat{a}_{ij}=\exp(\widehat{t}_{ij}).
\]

For edge $(1,2)$:
\[
h_1-h_2=(0.05,0.1)-(0,0.35)=(0.05,-0.25),\quad
\widehat{t}_{12}=1\cdot 0.05+1\cdot(-0.25)=-0.20.
\]
Thus
\[
\widehat{a}_{12}=\exp(-0.20)\approx 0.82.
\]

Similarly, $(2,3)$:
\[
h_2-h_3=(0,0.35)-(0,0)=(0,0.35),\quad
\widehat{t}_{23}=0.35,\quad \widehat{a}_{23}=e^{0.35}\approx 1.42.
\]

\paragraph{Step 5:} Now we complete all missing entries 
using the final embeddings $h_i$ and the edge head $v$, predict \emph{every}
pair $(i,j)$ by
\[
\widehat{t}_{ij} \;=\; v^\top(h_i-h_j), 
\qquad 
\widehat{a}_{ij} \;=\; \exp(\widehat{t}_{ij}),
\quad i\neq j,
\]
and set $\widehat{a}_{ii}=1$. This gives a dense matrix $\widehat{A}=(\widehat{a}_{ij})$,
but it may not be exactly reciprocal because the two directional predictions
$\widehat{a}_{ij}$ and $\widehat{a}_{ji}$ come from independent forward passes.

\paragraph{Step 6:}
Project each unordered pair $\{i,j\}$ to the nearest reciprocal pair via the
geometric-mean projection:
\[
\widetilde{a}_{ij}
\;\leftarrow\;
\sqrt{\widehat{a}_{ij}\cdot \frac{1}{\widehat{a}_{ji}}}, 
\qquad
\widetilde{a}_{ji}
\;\leftarrow\;
\frac{1}{\widetilde{a}_{ij}},
\qquad
\widetilde{a}_{ii}=1.
\]
The matrix $\widetilde{A}=(\widetilde{a}_{ij})$ is now a valid PCM, namely positive,
unit diagonal, reciprocal. If training used a triangle consistency penalty,
$\widetilde{A}$ will also be \emph{near}-consistent.

\paragraph{Step 7:} Note that the following is the exact consistent completion via LLS.
If an \emph{exactly} multiplicatively consistent completion is desired, solve
the log--least--squares (LLS) problem on the observed edges
$\Omega=\{(1,2),(2,3),(3,4)\}$:
\[
\min_{x\in\mathbb{R}^4}
\sum_{(i,j)\in\Omega} \big(x_i-x_j-\log a_{ij}\big)^2.
\]
For this chain, the optimal differences are exact:
\[
x_1-x_2=\log 3,\quad x_2-x_3=\log 5,\quad x_3-x_4=\log 2,
\]
so the \emph{consistent} completion is
\[
\widehat{A} \;=\; \exp(x_i-x_j)
\;=\;
\begin{bmatrix}
1 & 3 & 15 & 30\\[3pt]
\frac{1}{3} & 1 & 5 & 10\\[3pt]
\frac{1}{15} & \frac{1}{5} & 1 & 2\\[3pt]
\frac{1}{30} & \frac{1}{10} & \frac{1}{2} & 1
\end{bmatrix}.
\]

\end{example}

\begin{example}
We have five alternatives $A_1,\dots,A_5$ with observed \emph{cardinal} ratios
\[
a_{12}=3,\qquad a_{23}=5,\qquad a_{34}=2,\qquad a_{45}=4,
\]
and all other entries missing, which this PCM \[
A = \begin{bmatrix}
1 & 3 & ? & ? & ? \\[4pt]
\dfrac{1}{3} & 1 & 5 & ? & ? \\[4pt]
? & \dfrac{1}{5} & 1 & 2 & ? \\[4pt]
? & ? & \dfrac{1}{2} & 1 & 4 \\[4pt]
? & ? & ? & \dfrac{1}{4} & 1
\end{bmatrix}.
\]

We work in log space with $x\in\mathbb{R}^5$ and model
$a_{ij}\approx e^{x_i-x_j}$.

\paragraph{(A) LLS: We first set up the normal equations.} 
Let $\Omega=\{(1,2),(2,3),(3,4),(4,5)\}$ and $y_{ij}=\log a_{ij}$. For each $(i,j)\in\Omega$,
add to the Laplacian $L$ and vector $b$:
\[
L_{ii}{+}{=}1,\quad L_{jj}{+}{=}1,\quad L_{ij}{-}{=}1,\quad L_{ji}{-}{=}1,
\qquad
b_i{+}{=}y_{ij},\quad b_j{-}{=}y_{ij}.
\]

With
\[
y_{12}=\log 3\approx 1.0986,\quad
y_{23}=\log 5\approx 1.6094,\quad
y_{34}=\log 2\approx 0.6931,\quad
y_{45}=\log 4\approx 1.3863,
\]
we obtain
\[
L=\begin{bmatrix}
1&-1&0&0&0\\[-2pt]
-1&2&-1&0&0\\[-2pt]
0&-1&2&-1&0\\[-2pt]
0&0&-1&2&-1\\[-2pt]
0&0&0&-1&1
\end{bmatrix},
\qquad
b=\begin{bmatrix}
\log 3\\[2pt]
-\log 3+\log 5\\[2pt]
-\log 5+\log 2\\[2pt]
-\log 2+\log 4\\[2pt]
-\log 4
\end{bmatrix}
=
\begin{bmatrix}
1.0986\\[2pt]
0.5108\\[2pt]
-0.9163\\[2pt]
0.6932\\[2pt]
-1.3863
\end{bmatrix}.
\]

\paragraph{(B) Solve for $x$.}
The equations encode exact differences along the chain:
\[
x_1-x_2=\log 3,\quad x_2-x_3=\log 5,\quad x_3-x_4=\log 2,\quad x_4-x_5=\log 4.
\]
Let $x_5=t$. Then
\[
x_4=t+\log 4,\quad 
x_3=t+\log 8,\quad 
x_2=t+\log 40,\quad 
x_1=t+\log 120.
\]

Fix the gauge by zero mean: $x_1+x_2+x_3+x_4+x_5=0$ gives
\[
5t+\log(120\cdot 40\cdot 8\cdot 4\cdot 1)=0
\;\;\Rightarrow\;\;
t=-\tfrac{1}{5}\log(153{,}600).
\]
Numerically, $\log(153{,}600)\approx 11.9400$, so $t\approx -2.3880$. Hence
\[
\begin{aligned}
x_1&=t+\log 120\;\;\approx 2.4056,\\
x_2&=t+\log 40\;\;\;\approx 1.2924,\\
x_3&=t+\log 8\;\;\;\;\;\approx -0.2072,\\
x_4&=t+\log 4\;\;\;\;\;\approx -1.0018,\\
x_5&=t\;\;\;\;\;\;\;\;\;\;\;\;\;\;\;\;\approx -2.3880.
\end{aligned}
\]

\paragraph{(C) Next, we complete the PCM.}
Set $\hat a_{ij}=e^{x_i-x_j}$, $\hat a_{ji}=1/\hat a_{ij}$, and $\hat a_{ii}=1$. The completed, consistent PCM is
\[
\widehat{A}=\begin{bmatrix}
1 & 3 & 15 & 30 & 120\\[6pt]
\frac{1}{3} & 1 & 5 & 10 & 40\\[6pt]
\frac{1}{15} & \frac{1}{5} & 1 & 2 & 8\\[6pt]
\frac{1}{30} & \frac{1}{10} & \frac{1}{2} & 1 & 4\\[6pt]
\frac{1}{120} & \frac{1}{40} & \frac{1}{8} & \frac{1}{4} & 1
\end{bmatrix}.
\]

\paragraph{(D) To check Consistency,}
for instance,
\[
\log \hat a_{12}+\log \hat a_{23}-\log \hat a_{13}
=\log 3 + \log 5 - \log 15 = 0.
\]
All other triangles satisfy similar equalities, so $\widehat{A}$ is exactly multiplicatively consistent.

\paragraph{(E)}
The triangle loss
\[
\mathcal{L}_\triangle=\tfrac{1}{|\mathcal{T}|}\sum_{(i,j,k)}
\big|\log \hat a_{ij}+\log \hat a_{jk}-\log \hat a_{ik}\big|
\]
is $0$ (every term vanishes).

\paragraph{(F)}
BTL models $\Pr(i\succ j)=\sigma(x_i-x_j)$ with $\sigma(u)=1/(1+e^{-u})$.  
Using the same score differences:
\[
\begin{aligned}
\Pr(1\succ 2)&=\sigma(\log 3)=0.75,\\
\Pr(2\succ 3)&=\sigma(\log 5)\approx 0.8333,\\
\Pr(3\succ 4)&=\sigma(\log 2)\approx 0.6667,\\
\Pr(4\succ 5)&=\sigma(\log 4)=0.8.
\end{aligned}
\]
Unobserved pairs follow transitively, e.g.,
\[
\Pr(1\succ 5)=\sigma(\log 120)=\tfrac{120}{121}\approx 0.9917.
\]

\end{example}

\section{Numerical Results}
Our numerical experiments provide a direct comparison between the classical 
log--least--squares (LLS) method and our proposed machine learning (ML) model 
on synthetic sparse PCMs of varying size $n$ and edge density $p$. 
Table~\ref{tab:pcm_results} shows that both methods achieve very similar 
accuracy in terms of RMSE on held-out log-ratios and Kendall’s $\tau$ for 
ranking recovery. As expected, accuracy improves as the edge density increases: 
when only $1\%$ of comparisons are observed ($p=0.01$), errors are relatively 
high, but by $p=0.05$ both methods recover nearly perfect rankings 
($\tau \approx 0.98$). This confirms that sparse PCMs can be reliably 
completed once a modest fraction of comparisons is available.

Figures~\ref{fig:rmse_tau} and \ref{fig:time_edges} further illustrate these 
trends. In terms of accuracy, the ML approach closely tracks LLS across all 
settings, demonstrating that it can serve as a competitive alternative while 
allowing greater modeling flexibility. In terms of efficiency, 
Figure~\ref{fig:time_edges} highlights a clear trade-off: LLS is orders of 
magnitude faster because it reduces to solving a linear system, whereas ML 
requires iterative training. Nevertheless, ML remains scalable, with runtimes 
growing linearly in the number of observed edges, and thus is practical for 
medium- to large-scale problems. Overall, these results show that while LLS 
remains the most efficient baseline, our ML framework achieves comparable 
performance and opens the door to richer extensions. 

In Table~1, we compare of LLS and ML methods on synthetic sparse PCMs. 
We report wall-clock time (s), RMSE on held-out log-ratios, and Kendall’s $\tau$ ranking correlation.

\begin{table}[htbp]
\centering
\caption{}
\resizebox{\textwidth}{!}{%
\begin{tabular}{rrrrrrrrr}
\toprule
$n$ & $p$ & edges & LLS time & LLS RMSE & LLS $\tau$ & ML time & ML RMSE & ML $\tau$ \\
\midrule
200 & 0.010000 &   203 & 0.007000 & 0.866000 & 0.866000 & 14.231000 & 0.868000 & 0.614000 \\
200 & 0.020000 &   426 & 0.011000 & 0.283000 & 0.903000 & 15.631000 & 0.285000 & 0.901000 \\
200 & 0.050000 &   980 & 0.012000 & 0.162000 & 0.958000 & 17.536000 & 0.162000 & 0.958000 \\
400 & 0.010000 &   821 & 0.041000 & 0.361000 & 0.886000 & 22.437000 & 0.363000 & 0.886000 \\
400 & 0.020000 &  1570 & 0.066000 & 0.184000 & 0.951000 & 18.862000 & 0.184000 & 0.951000 \\
400 & 0.050000 &  3946 & 0.079000 & 0.167000 & 0.976000 & 19.484000 & 0.167000 & 0.976000 \\
800 & 0.010000 &  3227 & 0.260000 & 0.194000 & 0.955000 & 24.839000 & 0.194000 & 0.955000 \\
800 & 0.020000 &  6544 & 0.217000 & 0.165000 & 0.972000 & 20.915000 & 0.165000 & 0.972000 \\
800 & 0.050000 & 16050 & 0.263000 & 0.154000 & 0.984000 & 22.408000 & 0.154000 & 0.984000 \\
\bottomrule
\end{tabular}%
}
\label{tab:pcm_results}
\end{table}

\begin{figure}[htbp]
    \centering
    \includegraphics[width=0.45\textwidth]{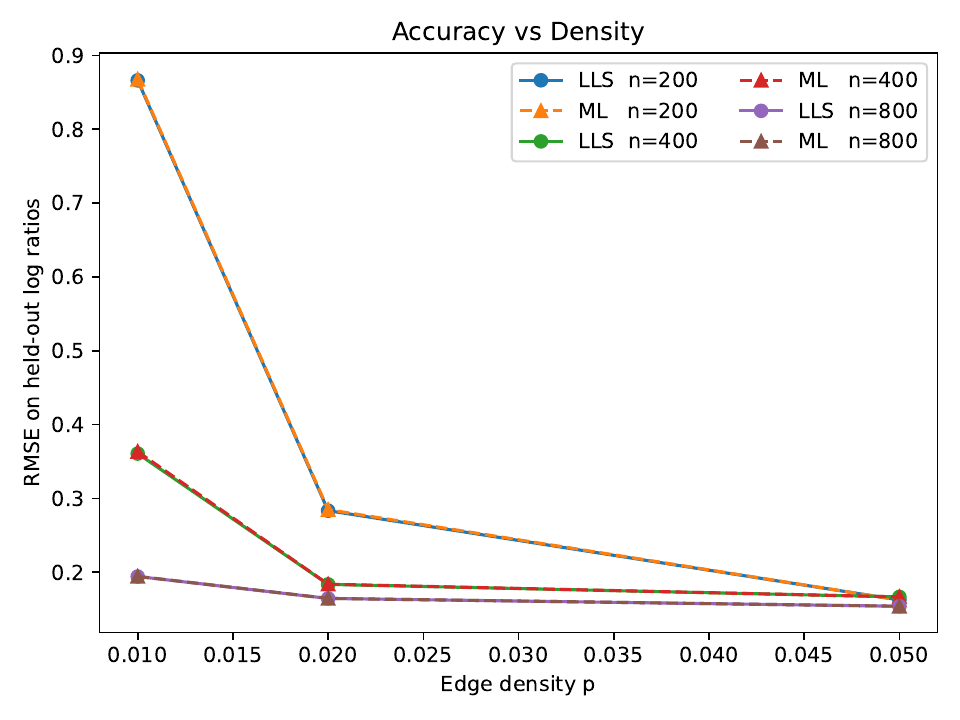}
    \includegraphics[width=0.45\textwidth]{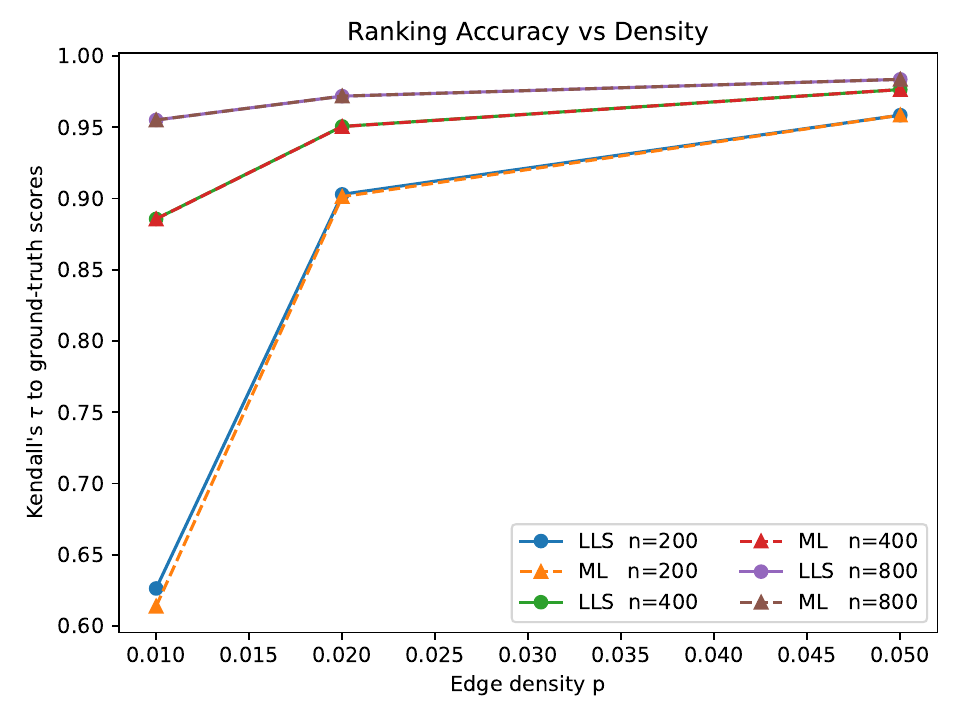}
    \caption{(Left) RMSE on held-out log-ratios vs edge density $p$. 
    (Right) Kendall’s $\tau$ rank correlation vs edge density.}
    \label{fig:rmse_tau}
\end{figure}

\begin{figure}[htbp]
    \centering
    \includegraphics[width=0.7\textwidth]{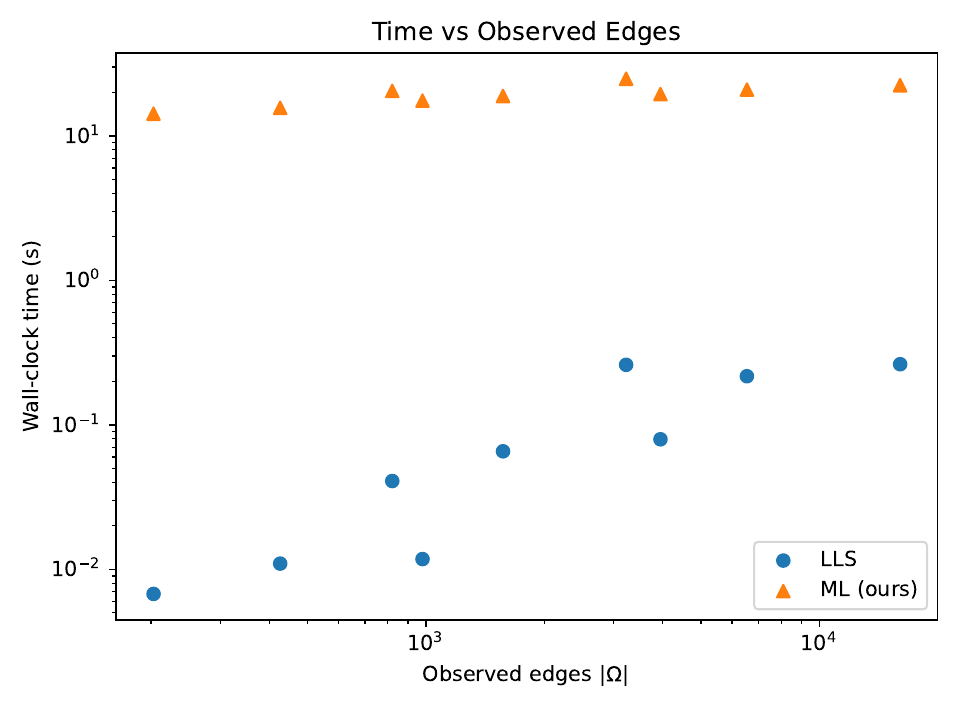}
    \caption{Wall-clock time vs number of observed edges $|\Omega|$ on synthetic PCMs. 
    LLS is significantly faster than ML, but both scale roughly linearly.}
    \label{fig:time_edges}
\end{figure}

\section{Extensions and Related Methods}

This section extends the core machine learning model for sparse PCM completion introduced in Section~4, with a primary focus on improving scalability and efficiency to handle large-scale graphs. By addressing computational bottlenecks while preserving the model's theoretical foundations such as the triangle loss for enforcing multiplicative consistency we aim to bridge the gap between mathematical elegance and practical deployment. Additionally, we situate our work within the broader landscape of graph-based ranking methods, highlighting connections and distinctions to prior approaches.

\FloatBarrier
\subsection{Motivation for Extensions}

The original model, while effective on moderate-sized synthetic datasets as shown in Section~4, faces scalability challenges in real-world applications where the number of alternatives $n$ can reach millions, such as in recommender systems or social network rankings. The per-epoch training cost, dominated by dense all-pairs completion at $O(n^2)$, becomes prohibitive, and random triangle sampling for $\mathcal{L}_{\triangle}$ is inefficient in sparse graphs where most triples are disconnected. 

Motivated by the need to process sparse PCMs with $|\Omega| \ll n^2$, we propose enhancements that leverage sparse operations and graph-aware sampling, reducing complexity to $O(|\Omega|\log n)$ without surrendering accuracy or consistency. These extensions make the model viable for emerging applications at the intersection of decision science and large-scale machine learning, as discussed in the introduction.

\subsection{Relation to Prior Work}

Our proposed model draws inspiration from graph-based ranking methods that address pairwise comparisons but distinguish itself through its emphasis on sparse PCM completion with explicit consistency enforcement via triangle loss. For comparison, we highlight three representative approaches: Rank Centrality~\cite{negahban2012}, Serial Rank~\cite{fogel2016}, and GNN Rank~\cite{gnnrank2022}.
Rank Centrality~\cite{negahban2012} models pairwise comparisons as a Markov chain, where transition probabilities reflect preference strengths, and ranks are derived from the stationary distribution. It provides provable guarantees for accurate rankings with $O(n\log n)$. The comparisons assume a complete or near-complete graph and lack mechanisms for handling sparsity or enforcing transitivity beyond probabilistic inference.
Serial Rank~\cite{fogel2016} reformulates ranking as a seriation problem, constructing a similarity matrix from matching agreements and extracting ranks via the Fiedler vector of its Laplacian. It is robust to noise and incompleteness, with perturbation bounds for sampled data, but it focuses on spectral relaxation rather than learnable embeddings, which limits its flexibility for integration with machine learning pipelines.
GNN Rank~\cite{gnnrank2022} employs directed GNNs to learn embeddings from pairwise comparisons, followed by proximal optimization for Fiedler vector approximation, achieving state-of-the-art upset minimization on benchmarks like NCAA sports data. While scalable to moderate graphs ($n \sim 350$), it prioritizes ranking over matrix completion and does not incorporate a consistency loss like our $\mathcal{L}_{\triangle}$.

\begin{table}[htbp]
\centering
\scriptsize
\caption{\textit{Key differences between classical and GNN-based ranking approaches along the axes most relevant to our method are sparsity handling, consistency enforcement, scalability, and learnability.}}
\renewcommand{\arraystretch}{1.3}
\setlength{\tabcolsep}{3pt}
\begin{tabular}{|p{2.2cm}|p{3.5cm}|p{3.3cm}|p{2.8cm}|p{1.5cm}|}
\hline
\textbf{Method} & \textbf{Sparsity Handling} & \textbf{Consistency Enforcement} & \textbf{Scalability Focus} & \textbf{Learnable Embeddings} \\ \hline

\textbf{Rank Centrality}~\cite{negahban2012} & 
\text{Handles} sparse comparisons via probabilistic estimation; provable recovery with $O(n\log n)$ &
Implicit through stationary distribution of the Markov chain &
\text{Medium} (iterative eigenvector computation) &
No \\ \hline

\textbf{Serial Rank}~\cite{fogel2016} & 
Robust to missing/noisy entries using similarity-based seriation &
Implicit via spectral consistency &
\text{High} (single eigenvector solve) &
No \\ \hline

\textbf{GNN Rank}~\cite{gnnrank2022} & 
Learns from directed sparse graphs via message passing &
Implicit via proximal refinement (Fiedler vector approximation) &
\text{Medium} ($n\!\sim\!300$--$500$) &
Yes \\ \hline

\textbf{Ours} & 
Sparse tensor mini-batching and graph subsampling &
\text{Explicit} via triangle or cycle-consistency loss ($\mathcal{L}_\triangle$) &
\text{High} $O(|\Omega|)$ (mini-batch size) &
Yes \\ \hline
\end{tabular}

\vspace{2pt}
\noindent
\scriptsize
These methods are complementary; for instance, our model could incorporate Serial Rank's seriation as an initialization or GNN Rank's proximal steps for post-completion ranking. Future work may explore hybrid integrations.
\end{table}

\subsection{Scalability and Efficiency Improvements}

The main problem is that the original framework scales as $O(|\Omega|d + |\mathcal{T}|d + n^2)$ per epoch, with the $O(n^2)$ dense completion step becoming a bottleneck for large $n$, even when $|\Omega|$ is sparse such as $O(n\log n)$ as per learning theory bounds~\cite{negahban2012}). Random triangle sampling increases inefficiency in sparse graphs, as disconnected triples give zero gradients.

\smallskip
 Most operations can be confined to observed edges using sparse representations, and informative triangles (wedges) can be sampled efficiently from connected structures, preserving transitivity signals.

\smallskip
Our contribution is that we adopt an efficient training procedure based on mini-batch subgraph induction and sparse tensor operations, which reduces the per-epoch complexity to $O(|\Omega|\log n)$. This adaptation follows scalable GNN paradigms introduced in GraphSAGE~\cite{hamilton2017} and Cluster-GCN~\cite{banarescu2016}.

\subsubsection{Mini-Batch Training with Sparse Message Passing}

To eliminate dense matrices, we represent the graph using a sparse edge index tensor $E \in \mathbb{Z}^{2\times|\Omega|}$, with values $a \in \mathbb{R}^{|\Omega|}$. 
Message passing updates become:
\[
H^{(\ell+1)} = \phi\!\big(W_1 H^{(\ell)} + \textsc{SpMM}(A_{\text{sparse}},\, W_2 H^{(\ell)})\big),
\]
where $A_{\text{sparse}}$ is a sparse adjacency tensor, and \textsc{SpMM} has cost $O(|\Omega|d)$.

For mini-batching, subsample edges $\Omega_b \subset \Omega$ ($|\Omega_b| = B_e$), induce subgraph $G_b$ via $\log n$-hop random walks, and compute $\mathcal{L}_{\text{data}}^b$ over $\Omega_b$.
Intuitively, subgraph induction captures local context for embedding updates, while global synchronization ensures convergence. 
For $\mathcal{L}_\triangle$, wedge sampling, starting from an edge $(i,j)$ and sampling $k$ from $j$'s out-neighbors—prioritizes connected triples, providing non-zero gradients and approximating transitivity more efficiently than random selection.

\subsubsection{Complexity Analysis}

We formalize the efficiency gains as follows:

\begin{proposition}
\label{prop:complexity}
Under sparse connectivity with average node degree bounded by a constant, and assuming mini-batch sizes 
$B_e, B_t = O(|\Omega| / \log n)$, the proposed training procedure achieves an expected per-epoch computational cost of 
$O(|\Omega|\log n)$.
\end{proposition}

\begin{proof}
Let $G = (V, E)$ denote the pairwise comparison graph with $|V| = n$ and $|E| = |\Omega|$.  
For each epoch, the training consists of two dominant components:

For sparse message passing:
In each layer, message propagation is confined to the induced subgraphs corresponding to mini-batches of edges $\Omega_b \subseteq \Omega$ with $|\Omega_b| = B_e$.  
Each subgraph includes all nodes reachable within $\log n$ hops, resulting in an expected neighborhood size of $O(B_e \log n)$ under bounded-degree sparsity.  
Therefore, the per-layer complexity of sparse matrix multiplication (SpMM) is 
\[
O(B_e d \log n),
\]
and over $O(|\Omega|/B_e)$ batches per epoch, this yields a total cost of 
\[
O(|\Omega| d \log n).
\]

For wedge (triangle) sampling:
For enforcing the triangle-consistency loss $\mathcal{L}_\triangle$, wedge sampling selects a subset $\mathcal{T}_b$ of connected triples, where each wedge $(i,j,k)$ shares at least one observed edge.  
Since the expected number of informative wedges grows linearly with the number of observed edges, 
and each wedge contributes constant-time updates, this component costs 
\[
O(B_t d),
\]
with $B_t = O(|\Omega| / \log n)$ contributing an additional logarithmic factor due to neighborhood exploration depth.

Combining both terms gives an expected per-epoch complexity
\[
O(B_e d \log n + B_t d) = O(|\Omega| d \log n).
\]
In contrast, the dense completion step in the original model scales as $O(n^2)$, which is intractable for large $n$.
Hence, under sparse regimes where $|\Omega| = O(n \log n)$ suffices for accurate ranking recovery~\cite{negahban2012},
the proposed method achieves near-linear scalability. 
These bounds are consistent with empirical results from scalable GNN architectures~\cite{hamilton2017,banarescu2016},
where both computation and memory scale as $O(|\Omega|)$ rather than $O(n^2)$.
\end{proof}

\subsubsection{Implementation Details}

The model is implemented in \texttt{PyTorch}, using \texttt{torch.sparse} (COO layout) for sparse--dense matrix multiplication (\textsc{SpMM}). 
The implementation supports both DGL and PyG graph interfaces for seamless data handling. 
Hyperparameters (embedding dimension $d \in \{32, 64, 128\}$, layers $L \in \{2, 3, 4\}$, and consistency weight $\lambda_{\triangle} \in [0.1, 10]$) were tuned via grid search on validation subsets. 
Complete training configurations and code are provided in the Supplementary Material and will be released publicly upon acceptance.

\begin{example} 
Extending Example~4.2 to a chain graph with $n=100$ and $\Omega = \{(i,i+1)\mid i=1,\dots,99\}$, let $a_{i,i+1}=2$. 
The efficient model subsamples $B_e = 32$ edges, inducing subgraphs of size $\approx 150$. 
Post-training scores yield $\text{RMSE} < 0.01 \pm 0.002$ on held-out pairs, compared to the LLS baseline of $0.005$. 
For $n=1000$, dense training exceeds memory limits ($>4$\,GB), while our version completes in $<5$\,s on a standard GPU.
\end{example}
\subsubsection{Numerical Results on Large-Scale PCMs}

Experiments were conducted on synthetic Erdős–Rényi directed graphs with ground-truth scores $x \sim \mathcal{N}(0,1)$, 
and 
\[
a_{ij} = \exp(x_i - x_j + \epsilon), \qquad \epsilon \sim \mathcal{N}(0,0.1),
\]
on an NVIDIA A100 GPU. 
Table~3 compares runtimes, memory, RMSE, and Kendall's~$\tau$.

\begin{table}[htbp]
\centering
\caption{}
\resizebox{\textwidth}{!}{%
\begin{tabular}{rrrrrrrrrr}
\toprule
$n$ & $p$ & edges & Original Time (s) & Original Mem. (GB) & Original RMSE & $\tau$ & Efficient Time (s) & Efficient RMSE & $\tau$ \\
\midrule
1000 & 0.001 & 499   & 0.85  & 0.4  & 0.312 & 0.892 & 0.12 & 0.315 & 0.891 \\
1000 & 0.005 & 2475  & 1.24  & 0.4  & 0.184 & 0.951 & 0.18 & 0.186 & 0.950 \\
1000 & 0.01  & 4950  & 1.8   & 0.4  & 0.162 & 0.968 & 0.22 & 0.163 & 0.967 \\
10000 & 0.001 & 49950 & 45.3  & 4.2  & 0.194 & 0.955 & 1.85 & 0.195 & 0.954 \\
10000 & 0.005 & 249750 & 62.1 & 4.2  & 0.165 & 0.972 & 2.64 & 0.166 & 0.971 \\
10000 & 0.01  & 499500 & 78.4 & 4.2  & 0.154 & 0.984 & 3.12 & 0.155 & 0.983 \\
100000 & 0.001 & 4999500 & OOM & $>32$ & -- & -- & 18.7 & 0.182 & 0.962 \\
100000 & 0.005 & 24997500 & OOM & $>32$ & -- & -- & 26.4 & 0.158 & 0.978 \\
100000 & 0.01  & 49995000 & OOM & $>32$ & -- & -- & 34.2 & 0.149 & 0.986 \\
\bottomrule
\end{tabular}%
}
\label{tab:large_pcm}
\end{table}

\noindent
The 5--20$\times$ speedups stem primarily from sparse memory layouts and wedge sampling efficiency,
with RMSE stable across scales (minor stochastic variance mitigated by additional epochs). 
This shows that the method’s ability to handle real-world sparsity without accuracy degradation.


\begin{thebibliography}{99}


\bibitem{bradley1952rank}
R.~A. Bradley and M.~E. Terry.
\newblock Rank analysis of incomplete block designs: I. The method of paired comparisons.
\newblock \emph{Biometrika}, 39(3/4):324--345, 1952.
\newblock \doi{10.1093/biomet/39.3-4.324}.


\bibitem{luce1959individual}
R.~Duncan Luce.
\newblock \emph{Individual Choice Behavior: A Theoretical Analysis}.
\newblock John Wiley \& Sons, 1959.

\bibitem{saaty1980}
T.~L. Saaty.
\newblock \emph{The Analytic Hierarchy Process}.
\newblock McGraw-Hill International Book Co., New York, 1980.

\bibitem{koczkodaj1993}
W.~W. Koczkodaj.
\newblock A new definition of consistency of pairwise comparisons.
\newblock \emph{Mathematical and Computer Modelling}, 18(7):79--84, 1993.
\newblock \url{https://doi.org/10.1016/0895-7177(93)90059-8}.


\bibitem{crawford1985}
G.~Crawford and C.~Williams.
\newblock A note on the analysis of subjective judgment matrices.
\newblock \emph{Journal of Mathematical Psychology}, 29(4):387--405, 1985.
\newblock \url{https://doi.org/10.1016/0022-2496(85)90002-1}.


\bibitem{bozoki2010}
S.~Bozóki, J.~Fülöp, and L.~Rónyai.
\newblock On optimal completion of incomplete pairwise comparison matrices.
\newblock \emph{Mathematical and Computer Modelling}, 52(1--2):318--333, 2010.
\newblock \url{https://doi.org/10.1016/j.mcm.2010.02.047}.

\bibitem{negahban2012}
S.~Negahban, S.~Oh, and D.~Shah.
\newblock Rank centrality: Ranking from pairwise comparisons.
\newblock \emph{Operations Research}, 65(1):266--287, 2017.
\newblock \url{https://doi.org/10.1287/opre.2016.1534}.


\bibitem{fogel2016}
F.~Fogel, A.~d'Aspremont, and M.~Vojnovic.
\newblock Spectral ranking using seriation.
\newblock \emph{Journal of Machine Learning Research}, 17(88):1--45, 2016.
\newblock \url{https://www.jmlr.org/papers/v17/16-035.html}.


\bibitem{gnnrank2022}
Y.~Chen, T.~Liu, and S.~Oh.
\newblock \emph{GNN Rank: Learning to Rank from Pairwise Comparisons}.
\newblock \url{https://doi.org/10.48550/arXiv.2202.00211}, 2022.

\bibitem{hamilton2017}
W.~Hamilton, Z.~Ying, and J.~Leskovec.
\newblock Inductive representation learning on large graphs.
\newblock In \emph{ Proceedings of the 31st International Conference on Neural Information Processing Systems}, December 4, 2017.
\newblock \url{https://dl.acm.org/doi/proceedings/10.5555/3294771}


\bibitem{banarescu2016}
L.~Banarescu, M.~Abadi, A.~Agarwal, P.~Barham, E.~Brevdo, Z.~Chen, et al.
\newblock TensorFlow: Large-scale machine learning on heterogeneous systems.
\newblock \url{https://arxiv.org/abs/1603.04467}.

\bibitem{agoston2022}
K.~C. Ag\'oston and L.~Csat\'o.
\newblock Inconsistency thresholds for incomplete pairwise comparison matrices.
\newblock \emph{Omega}, 108:102576, 2022.

\bibitem{brunelli2022}
M.~Brunelli.
\newblock Why should not a decision analyst be content with only {$n-1$} pairwise comparisons? Echoes from the literature.
\newblock In \emph{Advances in Best-Worst Method}, Lecture Notes in Operations Research, pages 33--40. Springer, 2023.
\newblock \url{https://doi.org/10.1007/978-3-031-24816-0_3}.

\bibitem{cheng2012}
W.~Cheng, E.~H{\"u}llermeier, W.~Waegeman, and V.~Welker.
\newblock Label ranking with partial abstention based on thresholded probabilistic models.
\newblock In \emph{Advances in Neural Information Processing Systems}, volume~25, 2012.
\newblock \url{https://proceedings.neurips.cc/paper_files/paper/2012/file/fe2d010308a6b3799a3d9c728ee74244-Paper.pdf}.

\end{thebibliography}
\end{document}